\documentclass[sigconf]{aamas}  % do not change this line!
\usepackage{balance}  % do not change this line -- unless you manually balance your last page

\usepackage{booktabs}

\usepackage{amssymb}
\usepackage{amsmath}
\usepackage{amsthm}
\usepackage{algorithm}
\usepackage{algorithmic}
\usepackage{graphicx}
\usepackage{subfigure}
\usepackage{color}
\usepackage{footmisc} 
\newtheorem{assumption}{Assumption}

\setcopyright{ifaamas}  % do not change this line!
\acmDOI{}  % do not change this line!
\acmISBN{}  % do not change this line!
\acmConference[AAMAS'19]{Proc.\@ of the 18th International Conference on Autonomous Agents and Multiagent Systems (AAMAS 2019)}{May 13--17, 2019}{Montreal, Canada}{N.~Agmon, M.~E.~Taylor, E.~Elkind, M.~Veloso (eds.)}  % do not change this line!
\acmYear{2019}  % do not change this line!
\copyrightyear{2019}  % do not change this line!
\acmPrice{}  % do not change this line!

\begin{document}

\title{TBQ($\sigma$): Improving Efficiency of Trace Utilization \\for Off-Policy Reinforcement Learning}  % put your title here!
%\titlenote{Titlenote --- e.g., used for things like "This article extends an earlier paper titled XYZ", and "equal contribution by the first two authors".}

% AAMAS: as appropriate, uncomment one subtitle line; see camera ready instructions
%\subtitle{Extended Abstract}
%\subtitle{Blue Sky Ideas Track}
%\subtitle{JAAMAS Track}
%\subtitle{Doctoral Consortium}                              
%\subtitle{Demonstration}
%\subtitlenote{Please refrain from using subtitle notes}

% replace this with your author block!
\author{Longxiang Shi}
\affiliation{%
  \institution{College of Computer Science and Technology, Zhejiang University}
  %\institution{Zhejiang University}
  \city{Hangzhou} 
  \state{Zhejiang Province, China} 
  \postcode{310007}
}
\email{shilongxiang@zju.edu.cn}
\author{Shijian Li}
\authornote{Corresponding author.\label{corresponding}}
\affiliation{%
  \institution{College of Computer Science and Technology, Zhejiang University}
  %\institution{Zhejiang University}
  \city{Hangzhou} 
  \state{Zhejiang Province, China} 
  \postcode{310007}
}
\email{shijianli@zju.edu.cn}
\author{Longbing Cao}
\affiliation{%
  \institution{Advanced Analytics Institute University of Technology Sydney}
  %\institution{}
  \city{Sydney} 
  \state{NSW, Australia}
  \postcode{2008}
}
\email{longbing.cao@uts.edu.au}
\author{Long Yang}
\affiliation{%
  \institution{College of Computer Science and Technology, Zhejiang University}
  %\institution{Zhejiang University}
  \city{Hangzhou} 
  \state{Zhejiang Province, China} 
  \postcode{310007}
}
\email{yanglong@zju.edu.cn}
\author{Gang Pan}
\affiliation{%
  \institution{College of Computer Science and Technology, Zhejiang University}
  %\institution{Zhejiang University}
  \city{Hangzhou} 
  \state{Zhejiang Province, China} 
  \postcode{310007}
}
\email{gpan@zju.edu.cn}

\begin{abstract}  % put your abstract here!
Off-policy reinforcement learning with eligibility traces is challenging because of the discrepancy between target policy and behavior policy. One common approach is to measure the difference between two policies in a probabilistic way, such as importance sampling and tree-backup. However, existing off-policy learning methods based on probabilistic policy measurement are inefficient when utilizing traces under a greedy target policy, which is ineffective for control problems. The traces are cut immediately when a non-greedy action is taken, which may lose the advantage of eligibility traces and slow down the learning process. Alternatively, some non-probabilistic measurement methods such as General Q($\lambda$) and Naive Q($\lambda$) never cut traces, but face convergence problems in practice. To address the above issues, this paper introduces a new method named TBQ($\sigma$), which effectively unifies the tree-backup algorithm and Naive Q($\lambda$). By introducing a new parameter $\sigma$ to illustrate the \emph{degree} of utilizing traces, TBQ($\sigma$) creates an effective integration of TB($\lambda$) and Naive Q($\lambda$) and continuous role shift between them. The contraction property of TB($\sigma$) is theoretically analyzed for both policy evaluation and control settings. We also derive the online version of TBQ($\sigma$) and give the convergence proof. We empirically show that, for $\epsilon\in(0,1]$ in $\epsilon$-greedy policies, there exists some degree of utilizing traces for $\lambda\in[0,1]$, which can improve the efficiency in trace utilization for off-policy reinforcement learning, to both accelerate the learning process and improve the performance.
\end{abstract}

% AAMAS: the ACM CCS are encouraged but optional within AAMAS papers
%%
%% The code below should be generated by the tool at
%% http://dl.acm.org/ccs.cfm
%\begin{CCSXML}
%<ccs2012>
%  <concept_id>10010520.10010553.10010562</concept_id>
%  <concept_desc>Computer systems organization~Embedded systems</concept_desc>
%  <concept_significance>500</concept_significance>
%  <concept_id>10010520.10010575.10010755</concept_id>
%  <concept_desc>Computer systems organization~Redundancy</concept_desc>
%  <concept_significance>300</concept_significance>
%  <concept_id>10010520.10010553.10010554</concept_id>
%  <concept_desc>Computer systems organization~Robotics</concept_desc>
%  <concept_significance>100</concept_significance>
%  <concept_id>10003033.10003083.10003095</concept_id>
%  <concept_desc>Networks~Network reliability</concept_desc>
%  <concept_significance>100</concept_significance>
%</ccs2012>  
%\end{CCSXML}
%
%\ccsdesc[500]{Computer systems organization~Embedded systems}

  % put your semicolon-separated keywords here!

\maketitle

%%%%%%%%%%%%%%%%%%%%%%%%%%%%%%%%%%%%%%%%%%%%%%%%%%%%%%%%%%%%%%%%%%%%%%%%%%%%%%%%%%%%%%%%%%%%%%%%%%%%%%%%%
%% start of main body of paper

\section{Introduction}

As a basic mechanism in reinforcement learning~(RL), eligibility traces~\cite{sutton:1988} unify and generalize temporal-difference~(TD) and Monte Carlo methods~\cite{sutton-barto:2011}. As a temporary record of an event~(e.g., taking an action or visiting a state) in RL, eligibility traces mark the memory parameters associated with the event as eligible for undergoing changes~\cite{sutton-barto:1998}. The eligible traces are then used to assign credit to the current TD-error which leads the learning of policies. With traces, credit is passed through multiple preceding states and therefore learning is often significantly faster~\cite{singh-dayan:1998}. 

With the on-policy TD learning with traces~(e.g., TD($\lambda$), Sarsa($\lambda$)), the assignment of credit to previous states decays exponentially according to the parameter $\lambda\in[0,1]$. If $\lambda=0$, the traces are set to zero immediately and the on-policy TD learning algorithm with traces is equal to one-step TD learning. If $\lambda=1$, the traces fade away slowly and no bootstrapping is made, and thus producing the Monte Carlo algorithm with online update~\cite{sutton-et-al:2014}. Moreover, the intermediate value of $\lambda$ makes the learning algorithm to perform better than the method at either extreme.

%The gap:

In the off-policy case, when the samples generated from a behavior policy is used to learn a different target policy, the usual approach is to measure the difference of the two policies in a probabilistic way. For example, Per-Decision Importance Sampling~\cite{precup-et-al:2000} weights returns based on the mismatch between target and behavior probabilities of the related actions. Alternatively, Tree-backup~(TB) algorithm~\cite{precup-et-al:2000} combines the value estimates for the actions along the traces according to their probabilities of target policy. More recently, Retrace($\lambda$)~\cite{munos:2016} combines Naive Q($\lambda$) with importance sampling, and offers a safe~(whatever the behavior policy is) and efficient~(can learn from full returns) way for off-policy reinforcement learning. However, existing off-policy learning methods based on state-action probability are inefficient when utilizing the traces for off-policy learning, especially when the target policy is deterministic, which is quite obvious in control problems. If the target policy is deterministic, the probability of target policy is zero when an exploratory action is taken. In this setting, importance sampling always involves a large variance since the importance ratio may be greater than 1 and is rarely used in practice. Retrace($\lambda$) and TB($\lambda$) is identical to Watkins' Q($\lambda$)~\cite{watkins-et-al:1989} and the traces are cut when an exploratory action is taken. This may cause to lose the advantage of eligibility traces and slow down the learning  process~\cite{sutton-barto:1998}. Peng's Q($\lambda$)~\cite{peng-et-al:1994} tried to solve this problem, but fails to converge to the optimal value. 

On the other hand, some existing methods do not depend on target policy probabilities and can learn from full returns without cutting traces under the greedy target policy. Unfortunately, some of them may face limitations in convergence. For instance, Naive Q($\lambda$)~\cite{sutton-barto:1998} never cuts traces thus provides a way to use full returns when performing off-policy RL with eligibility traces, which can sometimes achieve a better performance over Watkins' Q($\lambda$)~\cite{leng-et-al:2009}. A more recent work by \cite{harutyunyan:2016} shows that Naive Q($\lambda$) for control can converge to the optimal value under some conditions. An open question is: how about the intermediate condition between target policy probabilities-based and non-target policy probabilities-based methods?

To address the above question, in this paper we propose a TBQ($\sigma$) algorithm, which unifies TB($\lambda$)~(cutting traces immediately) and Naive Q($\lambda$)~(never cutting traces). By introducing a new parameter $\sigma$ to illustrate the \emph{degree} of utilizing traces, TBQ($\sigma$) creates a continuous integration and role shift between TB($\lambda$) and Naive Q($\lambda$). If $\sigma=1$ then TBQ($\sigma$) is converted to the Naive Q($\lambda$) that never cuts traces; and if $\sigma=0$ then TBQ($\sigma$) is transformed to the Watkins' Q($\lambda$). 
We then theoretically analyze the contraction property of TB($\sigma$) for both policy evaluation and control settings. We also derive the online version of TBQ($\sigma$) and give the convergence proof. Compared to TB($\lambda$), TBQ($\sigma$) is efficient in trace utilization with the greedy target policy. Compared to Naive Q($\lambda$), TBQ($\sigma$) can achieve convergence by adjusting a suitable $\sigma$.  We empirically show that, for $\epsilon\in (0,1]$ in $\epsilon$-greedy policies, there exists some degree of utilizing traces for $\lambda\in[0,1]$, which can improve the efficiency in trace utilization, therefore accelerating the learning process and improving the performance as well.

\section{Preliminaries and Problem Settings}

Here, we introduce some basic concepts, our target problems, notations, and related work. 

\subsection{Preliminaries and Problem Settings}
A reinforcement learning problem can be formulated as a Markovian Decision Process~(MDP) $(S,A,\gamma, P, r)$, where $S$ is a finite state space, $A$ is the action space, $\gamma{\in}[0,1]$ is the discount factor and $P$ is the mapping of transition function for each state-action pair $(s,a){\in}(S,A)$ to a distribution over $S$. A policy $\pi$ is a probability distribution over the set $(S{\times}A)$.

The state-action value $Q$ is a mapping on $S{\times}A$ to $\mathbb{R}$, which indicates the expected discounted future reward when taking action $a$ at state $s$ under policy $\pi$:
\begin{equation}Q(s,a):=\mathbb{E}_{\pi}(r_1+{\gamma}r_2+...+{\gamma}^{T-1}r_T|s_0=s,a_0=a)
\end{equation}
where $T$ is the time of termination. For each policy $\pi$, we define the operator $P^{\pi}$~\cite{harutyunyan:2016}:
$$(P^{\pi}Q)(s,a):=\sum\limits_{s'{\in}S}\sum\limits_{a'{\in}A}P(s'|s,a){\pi}(a'|s')Q(s',a')$$

For an arbitrary policy $\pi$ we use $Q^{\pi}$ to describe the unique Q-function corresponding to $\pi$:
$$Q^{\pi}:=\sum\limits_{t{\geq}0}{\gamma}^t(P^\pi)^tr$$

The Bellman operator $\mathcal{T}^{\pi}$ for a policy $\pi$ is defined as:
\begin{equation}
\mathcal{T}^{\pi}Q:=r+{\gamma}P^{\pi}Q
\end{equation}
Obviously, $\mathcal{T}^{\pi}$ has a unique fixed point $Q^{\pi}$:
\begin{equation}
\mathcal{T}^{\pi}Q^{\pi}=Q^{\pi}=(I-{\gamma}P^{\pi})^{-1}r
\end{equation}
The Bellman optimality operator $\mathcal{T}$ introduces a maximization over a set of policies and is defined as:
\begin{equation}
\mathcal{T}Q:=r+\gamma\max\limits_{\pi}P^{\pi}Q
\end{equation}
Its unique fixed point is $Q^*:=\sup_{\pi}{Q^{\pi}}$.

The Bellman equation can also be extended using the exponentially weighted sum of $n$-step returns~\cite{sutton:1988}:

\begin{equation}
\begin{aligned}
\mathcal{T}_{\lambda}^{\pi}&:=(1-\lambda)\sum\limits_{n{\geq}0}{\lambda}^n[(\mathcal{T}^\pi)^n Q]\\
&=Q+(I-\lambda\gamma{P}^{\pi})^{-1}(\mathcal{T}^{\pi}Q-Q)
\end{aligned}
\end{equation}

In this $\lambda$-return version of Bellman equation, the fixed point of $\mathcal{T}_{\lambda}^{\pi}$ is also $Q^{\pi}$. By varying the parameter $\lambda$ from 0 to 1, $\mathcal{T}_{\lambda}^{\pi}$ provides a continuous connection and role shift between one-step TD learning and Monte Carlo methods.

In this paper, we consider two types of RL problems, and mainly focus on action-value case under the off-policy setting. That is, in a \emph{policy evaluation} problem, we wish to estimate $Q^{\pi}$ of a fixed policy $\pi$ under the samples drawn from a different behavior policy $\mu$; in a \emph{control} problem, we seek to approximate $Q^*$ based on the iteration of Q-values. We specially focus on the learning scenario that the target policy is greedy, which is obvious in the control setting. Our main challenge is to improve the efficiency of trace utilization as well as ensure learning convergence during the off-policy learning process.

\subsection{Related Work}
%分两种case: probability case, non-probability case

Based on the usage of target policy probability when calculating the $\lambda$-return, existing works can be divided into 2 categories:

\subsubsection{Target policy probability-based methods.}
The $n$-step methods face challenges when involving off-policy, which has triggered to produce many methods to solve those challenges. The most common approach is to measure the two policies in a probabilistic sense\cite{meng-et-al:2018}. Based on the  work in \cite{munos:2016}, several off-policy return-based methods based on target policy probability: importance sampling (IS), tree-backup and Retrace($\lambda$) can be expressed in a unified operator $\mathcal{R}$ as follows:
\begin{equation}
\begin{aligned}
&\mathcal{R}Q(s,a):=Q(s,a)+\mathbb{E}_{\mu}[\sum\limits_{t{\geq}0}\gamma^t(\prod_{i=1}^{t}c_i)\delta_t]\\
&\delta_t=r_t+\gamma\mathbb{E}_{\pi}Q(s_{t+1},\cdot)-Q(s_t,a_t)
\end{aligned}
\end{equation}

\textbf{Importance sampling}: $c_i=\frac{\pi(a_i|s_i)}{\mu(a_i|s_i)}$. The IS methods correct the difference between target policy and behavior policy by their division of probabilities~\cite{sutton-barto:2011}. For example, Per-Decision Importance Sampling (PDIS)~\cite{precup-et-al:2000} incorporates eligibility traces with importance sampling. Since the estimation value contains a cumulative production of importance rations ($c_s$) which may exceeds 1, IS methods suffer from large variance and are seldom used in practice. In addition, weighted importance sampling~\cite{precup:2000} can reduce the variance of IS, but leads to a biased estimation.

\textbf{Tree-backup}: $c_i=\lambda\pi(a_i|s_i)$. The TB($\lambda$) algorithm~\cite{precup-et-al:2000} provides an alternative way for off-policy learning without IS. In control problems, if the target policy is greedy, then TB($\lambda$) produces Watkins' Q($\lambda$)~\cite{watkins-dayan:1992}. In this case, TB($\lambda$) is not efficient as it cuts traces when encountered an exploratory action and is not able to learn from the full returns. 

\textbf{Retrace($\lambda$)}: $c_i=\lambda\min(1,\frac{\pi(a_i|s_i)}{\mu(a_i|s_i)})$, was proposed in \cite{munos:2016}. Comparing to IS methods, this method truncates the importance ration by 1 to reduce the variance in IS. It is proved to convergence under any behavior policy and can learn from full returns when the behavior and target policies are near. However, in the control case when the target policy is greedy, Retrace($\lambda$) is identical to TB($\lambda$) and is not efficient in utilizing traces.

\subsubsection{Non-target policy probability-based methods.}
In addition, there are also some methods that does not depend on target policy probability, and can make full use of the traces:

\textbf{General Q($\lambda$)}: General Q($\lambda$)~\cite{seijen-et-al:2009}\cite{hasselt:2011} generalizes the on-policy Sarsa($\lambda$) using the following update equation:
\begin{equation*}
    \begin{aligned}
    Q(s_t,a_t)\leftarrow& Q(s_t,a_t)+\alpha[\sum\limits_{i\geq t}^{T}(\lambda\gamma)^{i-t}\delta_t+\mathbb{E}_{\pi}Q(s_{t+1},\cdot)\\
    -&Q(s_t,a_t)]\\
    \delta_t=&r_t+\gamma\mathbb{E}_{\pi}Q(s_{t+1},\cdot)-\mathbb{E}_{\pi}Q(s_t,a_t)
    \end{aligned}
\end{equation*}
In control case, when target policy is greedy, General Q($\lambda$) is identical to Peng's Q($\lambda$)~\cite{peng-et-al:1994}. It does not cut traces so much as Watkins' Q($\lambda$). However, When learning is off-policy, General Q($\lambda$) lead to a biased estimation and does not converge to $Q^{\pi}$.

\textbf{Q($\lambda$) with off policy corrections} \cite{harutyunyan:2016}: it is an off-policy correction method based on a Q-baseline. Their proposed operator $\mathcal{R}_{\lambda}^{\pi,\mu}$ is the same as $\mathcal{R}$ if $c_i=\lambda$ in (6). Their algorithms, named $Q^{\pi}(\lambda)$ and $Q^*(\lambda)$ for policy evaluation and control, respectively. If the distance $d=\max\limits_s\lVert\pi(\cdot|s)-\mu(\cdot|s)\rVert$ between target policy $\pi$ and behavior policy $\mu$ is small, i.e., $d<\frac{1-\gamma}{\gamma}$, $Q^{\pi}(\lambda)$ converges to its fixed point $Q^{\pi}$. In control scenarios, $Q^*(\lambda)$ is equal to Naive $Q(\lambda)$~\cite{sutton-barto:1998} and is guaranteed to converge to $Q^*$ under $\lambda<\frac{1-\gamma}{2\gamma}$. Besides, they also empirically show that in fact there exists some trade-off between $d$ and $\lambda$ beyond the convergence guarantee, which can make the learning faster and better. In addition, $Q^{\pi}(\sigma,\lambda)$ is proposed in~\cite{yang-et-al:2018} to combine Sarsa($\lambda$) and Q$^{\pi}(\lambda)$, and inherit the similar properties with Q$^\pi$($\lambda$).

In conclusion, existing off-policy learning methods based on target policy probability are inefficient when utilizing eligibility traces, especially when target policy is greedy. In this scenario, The traces are cut immediately when encountered an exploratory action and thus may lose the advantage of eligibility traces and slow down the learning process. In addition, existing non-target policy probability based methods can make full use of the traces, but may face limitations in convergence. In this paper, we try to solve this dilemma by create a hybridization of those two different methods.

\section{TBQ($\sigma$): Degree of Traces Utilization}

In the RL literature, unifying different algorithmic ideas to leverage the pros and cons in each idea and to produce better algorithms has been a pragmatic approach~\cite{deasis:2017}. This also applies to several policy learning methods, e.g., TD($\lambda$) to unify TD-learning and Monte Carlo methods, Q($\sigma$)~\cite{deasis:2017} to fuse multi-step tree-backup and Sarsa, and Q($\sigma$,$\lambda$)~\cite{yang-et-al:2018} to integrate $Q^\pi(\sigma)$ and Sarsa($\lambda$). Such hybridization is useful for balancing the capabilities of different trace-cutting methods discussed above. Accordingly, in this paper, we introduce a new parameter $\sigma$ into trace-cutting to enable the degree of utilizing traces. The proposed method, TBQ($\sigma$), unifies TB($\lambda$)~(cutting traces immediately) and Naive Q($\lambda$)~(never cutting traces). 

We first give the definition of operator that used for the update equation of TBQ($\sigma$):
\begin{definition}
The proposed operator $\mathcal{R}_{\sigma}$ is a map on $\mathbb{R}^{|S|\times{|A|}}$ to $\mathbb{R}^{|S|\times{|A|}}$, ${\forall}s{\in}S,a{\in}A,\sigma\in[0,1]:$\\
\begin{equation}
\begin{aligned}
\mathcal{R}_{\sigma}:&\mathbb{R}^{|S|\times{|A|}}\leftarrow\mathbb{R}^{|S|\times{|A|}}\\
&Q(s,a):=Q(s,a)+\mathbb{E}_{\mu}[\sum\limits_{t{\geq}0}\gamma^t(\prod_{i=1}^{t}c_i)\delta_t]
\end{aligned}
\end{equation}

where
$$c_i=\lambda[\sigma+(1-\sigma)\pi(a_i|s_i)]$$
$$\delta_t=r_t+\gamma\mathbb{E}_{\pi}Q(s_{t+1},\cdot)-Q(s_t,a_t)$$
\end{definition}

TBQ($\sigma$) linearly combines TB($\lambda$) and Naive Q($\lambda$) by using the degree parameter $\sigma$. When $\sigma=0$ then TBQ($\sigma$) is converted to TB($\lambda$), and $\sigma=1$ TBQ($\sigma$) is transformed to Naive Q($\lambda$). By exploratory adjusting the parameter $\sigma$ from 0 to 1 we can produce a continuous integration and role shift between cutting the traces immediately and never cutting traces. We then analyze the contraction property of $\mathcal{R}_{\sigma}$ in policy evaluation. We here use $\lVert\cdot\rVert$ to represent the supremum norm.

\begin{theorem}
The proposed operator $\mathcal{R}_{\sigma}$ has a unique fixed point $Q^{\pi}$. If the behavior policy and target policy are near, i.e.,

$d=\max\limits_x{\lVert\pi(\cdot|x)-\mu(\cdot|x)\rVert}<(1-\gamma)[\frac{1}{\gamma\lambda}+1-\sigma]$, then $\lVert\mathcal{R}_{\sigma}Q-Q^{\pi}\rVert=O(\eta^k)$.
\end{theorem}

\begin{proof}
Unfolding the operator:
\begin{equation*}
\begin{aligned}
\mathcal{R}_{\sigma}Q-Q^{\pi}&=\sigma(\mathcal{R}_{\lambda}^{\pi,\mu}Q-Q^{\pi})+(1-\sigma)(\mathcal{R}Q-Q^{\pi})\\
\end{aligned}
\end{equation*}

Taking the supremum norm:
\begin{equation*}
\begin{aligned}
\lVert\mathcal{R}_{\sigma}Q-Q^{\pi}\rVert&=\lVert\sigma(\mathcal{R}_{\lambda}^{\pi,\mu}Q-Q^{\pi})+(1-\sigma)(\mathcal{R}Q-Q^{\pi})\rVert\\
&\leq\sigma\lVert\mathcal{R}_{\lambda}^{\pi,\mu}Q-Q^{\pi}\rVert+(1-\sigma)\lVert\mathcal{R}Q-Q^{\pi}\rVert\\
\end{aligned}
\end{equation*}

Per Lemma 1 in \cite{harutyunyan:2016} we have:
$$\lVert\mathcal{R}_{\lambda}^{\pi,\mu}Q-Q^{\pi}\rVert\leq\frac{\gamma(1-\lambda+\lambda d)}{1-\lambda\gamma}\lVert Q-Q^{\pi}\rVert$$

where $d$ is the distance between $\pi$ and $\mu$:

$$\max\limits_{x}\lVert\pi(\cdot|x)-\mu(\cdot|x)\rVert\leq d$$

Per Theorem 1 in \cite{munos:2016} we have:
$$\lVert\mathcal{R}Q-Q^{\pi}\rVert\leq\gamma\lVert Q-Q^{\pi}\rVert$$

Adding the above two items we have:
$$\lVert\mathcal{R}_{\sigma}Q-Q^{\pi}\rVert\leq\eta\lVert Q-Q^{\pi}\rVert$$

where $\eta=\frac{\gamma-\lambda\gamma^2+\lambda\sigma\gamma^2+\sigma\gamma\lambda d-\sigma\gamma\lambda}{1-\gamma\lambda}$.

Further, for $d<(1-\gamma)[\frac{1}{\gamma\lambda}+1-\sigma]$, $\eta<1$, we have $$\lVert\mathcal{R}_{\sigma}Q-Q^{\pi}\rVert=O(\eta^k)$$
\end{proof}

Theorem 3.2 indicates that, for any $\lambda\in[0,1]$, if the distance between two policies are near with regard to $\sigma$, then $Q_k$ converges to $Q^\pi$. Comparing to $Q^{\pi}(\lambda)$~\cite{harutyunyan:2016}, our algorithm derives a wider convergence range w.r.t $\sigma$. We provide a hybridization of utilizing traces based on TB($\lambda$) and Naive Q($\lambda$). In practice, the convergence condition can be satisfied by adjusting the parameter $\sigma$ under different situations.

\section{TBQ($\sigma$) for Control}
In control problems, we want to estimate $Q^*$ by iteratively applying policy evaluation and policy improvement processes, which is referred to \emph{generalized policy iteration}~(GPI)~\cite{sutton-barto:1998}. Denoting $(Q_k,\pi_k)$ as the Q-value and the corresponding target policy in the iteration process under the arbitrary behavior policy $\mu_k$ at step $k$, then $\pi_{k+1}$ can be retrieved by our operator $\mathcal{R_{\sigma}^{\pi,\mu}}$ by using the following steps:
\begin{itemize}
\item Policy evaluation step:
$$Q_{k+1}=\mathcal{R}_{\sigma}^{\pi_k,\mu_k}Q_k$$
\item Policy improvement step:
$$\pi_{k+1}=greedy(Q_{k+1})$$
\end{itemize}

We here use the notion $greedy(Q_{k})$ to represent $\pi_{k}$, which is greedy with respect to $Q_{k}$. Based on GPI, the TBQ($\sigma$) algorithm for control problems is depicted in Algorithm 1 with an online forward view, i.e., TBQ$^F$($\sigma$). Note that $\mathbb{I}\{(s_t,a_t)=(s,a)\}$ is the indicator function.

To analyze the convergence of Algorithm 1, we first consider off-line version of the TBQ($\sigma$) algorithm. The following lemma states that, if $\sigma$ satisfies some condition with regard to $\lambda$, then the off-line version of TBQ($\sigma$) is guaranteed to converge.

\begin{lemma}
Considering the sequence $\{(Q_k,\pi_k)\}_{k\geq 0}$ generated by the operator $\mathcal{R}_{\sigma}$ under a greedy target policy $\pi_k$ and an arbitrary behavior policy $\mu$, we have:
$${\lVert}Q_{k+1}-Q^*{\rVert}{\leq}{\eta}{\lVert}Q_{k}-Q^*{\rVert}$$
where $\eta=\frac{\sigma\gamma+\sigma\lambda\gamma}{1-\lambda\gamma}+(1-\sigma)\gamma$.

Specifically, if $\lambda\leq\frac{1-\gamma}{\sigma\gamma
+\sigma\gamma^2+\gamma-\gamma^2}$, then the sequence $\{Q_k\}_{k\geq 1}$ converges to $Q^*$ exponentially fast.
\end{lemma}

\begin{proof}

Unfolding the operator:
\begin{equation*}
\begin{aligned}
{\lVert}Q_{k+1}-Q^*{\rVert}&={\lVert}\mathcal{R}_{\sigma}Q_k-Q^*{\rVert}\\
&\leq\sigma{\lVert}\mathcal{R}_{\lambda}^{{\pi},{\mu}}Q_k-Q^*{\rVert}+(1-\sigma){\lVert}\mathcal{R}Q_k-Q^*{\rVert}
\end{aligned}
\end{equation*}

based on \cite{harutyunyan:2016} and \cite{munos:2016}, we have:
$${\lVert}\mathcal{R}_{\lambda}^{{\pi},{\mu}}Q_k-Q^*{\rVert}\leq\frac{\gamma+\lambda\gamma}{1-\lambda\gamma}{\lVert}Q_k-Q^*{\rVert}$$
$${\lVert}\mathcal{R}Q_{k}-Q^*{\rVert}\leq\gamma{\lVert}Q_k-Q^*{\rVert}$$

As a consequence, we deduce the result:
\begin{equation*}
\begin{aligned}
{\lVert}Q_{k+1}-Q^*{\rVert}&={\lVert}\mathcal{R}_{\sigma}Q_k-Q^*{\rVert}\\
&\leq\sigma{\lVert}\mathcal{R}_{\lambda}^{{\pi},{\mu}}Q_k-Q^*{\rVert}+(1-\sigma){\lVert}\mathcal{R}Q_k-Q^*{\rVert}\\
&\leq\sigma\frac{\gamma+\lambda\gamma}{1-\lambda\gamma}{\lVert}Q_k-Q^*{\rVert}+(1-\sigma)\gamma{\lVert}Q_k-Q^*{\rVert}\\
&=[\frac{\sigma\gamma+\sigma\gamma\lambda}{1-\lambda\gamma}+(1-\sigma)\gamma]{\lVert}Q_k-Q^*{\rVert}
\end{aligned}
\end{equation*}
\end{proof}

Lemma 4.1 states that, for any $d$, if $\lambda\leq\frac{1-\gamma}{\sigma\gamma
+\sigma\gamma^2+\gamma-\gamma^2}$ then the off-line control algorithm is guaranteed to converge. However, similar to  $Q^*(\lambda)$~\cite{harutyunyan:2016}, in practice, there exist some trade-offs between $\lambda$ and $\sigma$ under different $d$ values, which goes beyond the convergence guarantee. By introducing a new parameter $\sigma$, we can alleviate $\lambda-d$ relationship through adjusting a suitable $\sigma$. The traces can also be utilized when an exploratory action is taken. In addition, comparing to Naive Q($\lambda$), we derive a wider convergence range by tuning $\sigma$. Although we have not give a detail theoretical analyze of $\lambda-d$ relationship under different $\sigma$, in the experiment part we will show that for any $\lambda\in[0,1]$ and $\epsilon\in[0,1]$ in $\epsilon-greedy$ policies, there exist some degree of utilizing traces $\sigma$, which can accelerate the learning process and yield a better performance through utilizing the full returns as well.

\begin{algorithm}[H]
\caption{TBQ$^F$($\sigma$): The online forward view version of TBQ($\sigma$) algorithm}
\begin{algorithmic}
\STATE {\textbf{Input}: discounting factor $\gamma$, degree of utilizing traces $\sigma$, bootstrapping parameter $\lambda$, and stepsize $\alpha_k$}
\STATE {\textbf{Initialization}: $Q_0(s,a)$ arbitrary}
\FOR{Episode $k$ from $1$ to $n$}%针对每个episode
\STATE $Q_{k+1}(s,a)\leftarrow Q_k(s,a)$ $\forall (s,a)$
\STATE $e(s,a)\leftarrow 0$ $\forall (s,a)$
\STATE Sample a trajectory $s_0,a_0,r_0,...,x_{T_k}$ from $\mu_k$
\FOR{Sample $t$ from $0$ to $T_k-1$}%针对采样的样本
\STATE $\delta_t^{\pi_k}\leftarrow r_t+\gamma\max\limits_{a'}Q_{k+1}(s_{t+1},a')-Q_{k+1}(s_t,a_t)$
\STATE 
$c_t=
\begin{cases}
1&t=0\\
\sigma+(1-\sigma)\pi(a_t|s_t)&t\neq 0
\end{cases}$
\STATE $e(s,a)\leftarrow\lambda\gamma c_t e(s,a)+\mathbb{I}\{(s_t,a_t)=(s,a)\}$ $\forall (s,a)$
\STATE $Q_{k+1}\leftarrow Q_{k+1}+\alpha \delta_t^{\pi_k} e(s,a)$ $\forall (s,a)$
\ENDFOR
\ENDFOR
\end{algorithmic}
\end{algorithm}

\subsection{Convergence Analysis of TBQ($\sigma$) Algorithm}
We now consider the convergence proof of TBQ($\sigma$) described in Algorithm 1. First, we make some assumptions similar to~\cite{harutyunyan:2016}~\cite{munos:2016}.

\begin{assumption}
For bounded stepsize $\alpha_k$:
$\sum\limits_{k\geq 0}\alpha_k(s,a)=\infty$, 

$\sum\limits_{k\geq 0}\alpha_k(s,a)<\infty$.
\end{assumption}
\begin{assumption}
Minimum visit frequency: all $(s,a)$ pairs are visited infinitely often:
$\sum\limits_{t\geq 0}P\{(s_t,a_t)=(s,a)\}\geq D>0$.
\end{assumption}
\begin{assumption}
Finite sample trajectories: $\mathbb{E}_{\mu_k}(T_k^2)<\infty$, $T_k$ denotes the length of sample trajectories.
\end{assumption}

Under those assumptions, Algorithm 1 can converge to $Q^*$ with probability 1 as stated below:
\begin{theorem}
Considering the sequence of Q-functions $\{(Q_k,\pi_k)\}_{k\geq 0}$ generated from Algorithm 1, where $\pi_k$ is the greedy policy with respect to $Q_k$, if $\lambda\leq\frac{1-\gamma}{\sigma\gamma
+\sigma\gamma^2+\gamma-\gamma^2}$, then under Assumptions 1-3, $Q_k\rightarrow Q^*$ with probability 1.
\end{theorem} 

\begin{proof}
For reading convenience, we first define some notations:
Let $k$ denote the $k$th iteration, $t$ denote the length of the trajectory, $l$ denote the $l$th sample of current trajectory, then the accumulating trace~\cite{sutton-barto:1998} $z_{l,t}^k$ can be written as:
\begin{equation}
z_{l,t}^k=\sum\limits_{j=k}^t\gamma^{t-j}(\prod\limits_{i=j+1}^t c_i)\mathbb{I}\{(s_j,a_j)=(s_l,a_l)\}
\end{equation}

We use $Q_{k}^o(s_l,a_l)$ to emphasize the online setting, then Equation (7) can be written as:
\begin{equation}
Q_{k+1}^o(s_l,a_l)\leftarrow Q_k^o(s_l,a_l)+\alpha_k(s_l,a_l)\sum\limits_{t\geq l}\sigma_t^{\pi_k}z_{l,t}^k
\end{equation}
\begin{equation}
\delta_t^{\pi_k}=r_t+\gamma\mathbb{E}_{\pi_k}Q_k^o(s_{t+1},\cdot)-Q_k^o(s_t,a_t)
\end{equation}

Since $c_i=\lambda[\sigma+(1-\sigma)\pi(a_i|s_i)]\leq 1$, based on Assumption 3, we have:
\begin{equation}
\mathbb{E}[\sum\limits_{t\geq l}z_{l,t}^k]<\mathbb{E}[T_k^2]<\infty
\end{equation}

Therefore, the total update is bounded based on Equation (11). Further, we can rewrite the update Equation (9) as:

\begin{equation*}
\begin{aligned}
Q_{k+1}^o(s_l,a_l)&\leftarrow(1-D_k\alpha_k)Q_k^o(s_l,a_l)+D_k\alpha_k(\mathcal{R}_{\sigma}^{\pi_k}Q_k^o(s_l,a_l)\\
&+w_k(s_l,a_l)+v_k(s_l,a_l))\\
w_k(s_l,a_l)&:=(D_k)^{-1}[\sum\limits_{t\geq l}\delta_t^{\pi_k}z_{l,t}^k-\mathbb{E}_{\mu_k}(\sum\limits_{t\geq l}\delta_t^{\pi_k}z_{l,t}^k)]\\
v_k(s_l,a_l)&:=(D_k\alpha_k)^{-1}(Q_{k+1}^o(s_l,a_l)-Q_{k+1}(s_l,a_l))\\
D_k&:=D_k(s_l,a_l)=\sum\limits_{t\geq l}P\{(s_t,a_t)=(s_l,a_l)\}\\
\end{aligned}
\end{equation*}

Based on Assumptions 1 and 2, the new stepsize $(D_k\alpha_k)$ satisfies Assumption (a) of Proposition 4.5 in \cite{bertsekas-tsitsiklis:1996}. Lemma 4.1 states that the operator $\mathcal{R}_{\sigma}$ is a contraction, which satisfies Assumption (c) of Proposition 4.5 in~\cite{bertsekas-tsitsiklis:1996}. Based on Equation (7) and the bounded reward function, the variance noise term $w_k$ is bounded, thus Assumption (b) of Proposition 4.5 in \cite{bertsekas-tsitsiklis:1996} is satisfied. The noise term $v_k$ can also be shown to satisfy Assumption (d) of Proposition 4.5 in \cite{bertsekas-tsitsiklis:1996}, based on Proposition 5.2 in \cite{bertsekas-tsitsiklis:1996}. Finally, we are able to apply Proposition 4.5~\cite{bertsekas-tsitsiklis:1996} to conclude that the sequence $Q_k^o$ converges to $Q^*$ with probability 1.
\end{proof}

\subsection{Online Backward Version of TBQ($\sigma$)}

Since the online forward view algorithm described in Algorithm 1 needs extra memory to store the trajectories, we here also provide an online backward version of TBQ($\sigma$): TBQ$^B$($\sigma$). Based on the equivalence between forward view and backward view of the eligibility traces~\cite{sutton-barto:1998}, the online backward view version of TBQ($\sigma$) can be implemented as in Algorithm 2. The online backward view version TBQ$^B$($\sigma$) provides a more concise and efficient form and it is more efficient in executing the TBQ($\sigma$) algorithm.

\begin{algorithm}[H]
\caption{TBQ$^B$($\sigma$): On-line backward version of TBQ($\sigma$) algorithm}
\begin{algorithmic}
\STATE {\textbf{Input}: discounting factor $\gamma$, degree of cutting traces $\sigma$, bootstrapping parameter $\lambda$ and stepsize $\alpha$}
\STATE {\textbf{Initialization}: $Q(s,a)$ arbitrary}
\FOR{$k$ from $1$ to $n$}
\STATE Initialize $s,a$
\STATE $e(s,a)=0$ $\forall (s,a)$
\REPEAT
\STATE Take action $a$, observe state $s'$ and receive reward $r$
\STATE Choose $a'$ from $s'$ using $\epsilon$-greedy policy $\mu$ based on $Q(s,a)$
\STATE $a^*\leftarrow\arg\max\limits_b Q(s',b)$
\STATE $\delta\leftarrow r+\gamma\max\limits_{b}Q(s',b)-Q(s,a)$
\STATE $e(s,a)\leftarrow e(s,a)+1$
\FOR{all $s,a$}
\STATE $Q(s,a)\leftarrow Q(s,a)+\alpha\delta e(s,a)$
\IF{$a^*=a'$}
\STATE $e(s,a)\leftarrow \gamma\lambda e(s,a)$
\ELSE
\STATE $e(s,a)\leftarrow \sigma\gamma\lambda e(s,a)$ 
\ENDIF
\ENDFOR
\STATE $s\leftarrow s'$,$a\leftarrow a'$
\UNTIL $s$ is terminal
\ENDFOR
\end{algorithmic}
\end{algorithm}

\section{Experiments}
In this section, we explore the $\lambda-\sigma$ trade-off in the control case w.r.t. several environments. We empirically find that, for $\lambda\in[0,1]$ and $\epsilon\in[0,1]$, there exists some degree of utilizing traces $\sigma$, which can improve the efficiency of trace utilization.

\subsection{19-State Random Walk}
The 19-state random walk problem is a one-dimensional MDP environment which is widely used in RL~\cite{sutton-barto:2011}\cite{deasis:2017}. There are two terminal states at the two ends of the environment, transition to the left terminal receives a reward 0 and to the right terminal receives 1. The agent at each state has two actions: left and right. We here apply the online forward version TBQ$^F$($\sigma$) by using an $\epsilon-greedy$ policy as behavior policy and a greedy policy as target policy. For each episode, the maximum step is bounded as 100. We then measure the mean-squared-error~(MSE) of the optimal Q-value $Q^*$ between the estimated values and the analytically computed values after 10,000 episodes of offline running. We test 3 different $\epsilon$ values: 0.1, 0.5, 1. The corresponding distance $d$ between target policy and behavior policy is 0.05, 0.25, 0.5, respectively. For each $\epsilon$, we test different $\lambda$ values from 0 to 1 with stepsize 0.1. Also, for each $\lambda$, we also try different $\sigma$ values from 0 to 1 with stepsize 0.1. The learning stepsize $\alpha$ is tuned to 0.3. All results are averaged across 10 independent runs with fixed random seed. We compare TBQ($\sigma$) with TB($\lambda$) and Naive Q($\lambda$). For TBQ($\sigma$), we also mark out the best performance of $\sigma$, with the results shown in Figure 1. 

\begin{figure}
	\centering
	\subfigure
	{
	\includegraphics[width = .45\textwidth]{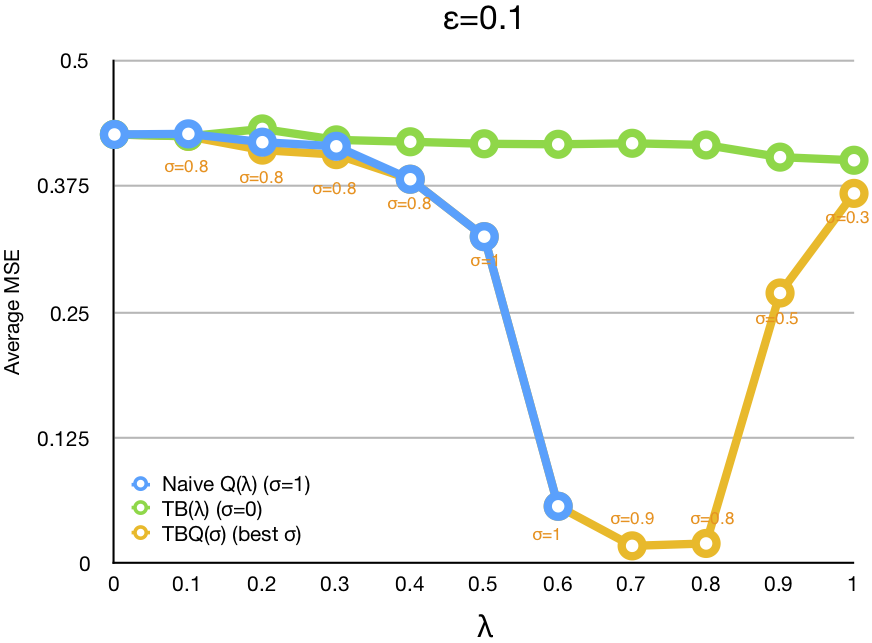}
	}
	\subfigure
	{	
	\includegraphics[width = .45\textwidth]{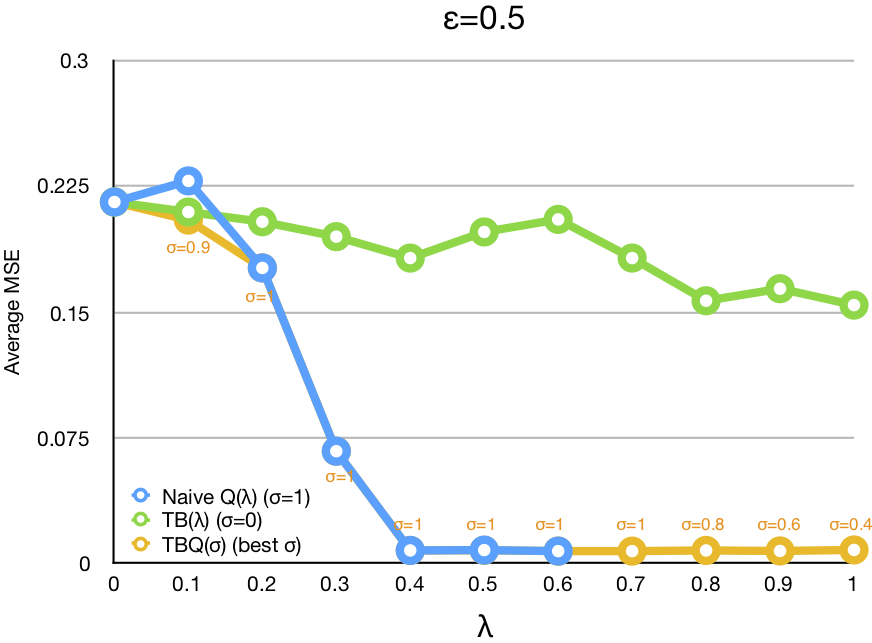}
	}
	\subfigure
	{	
	\includegraphics[width = .45\textwidth]{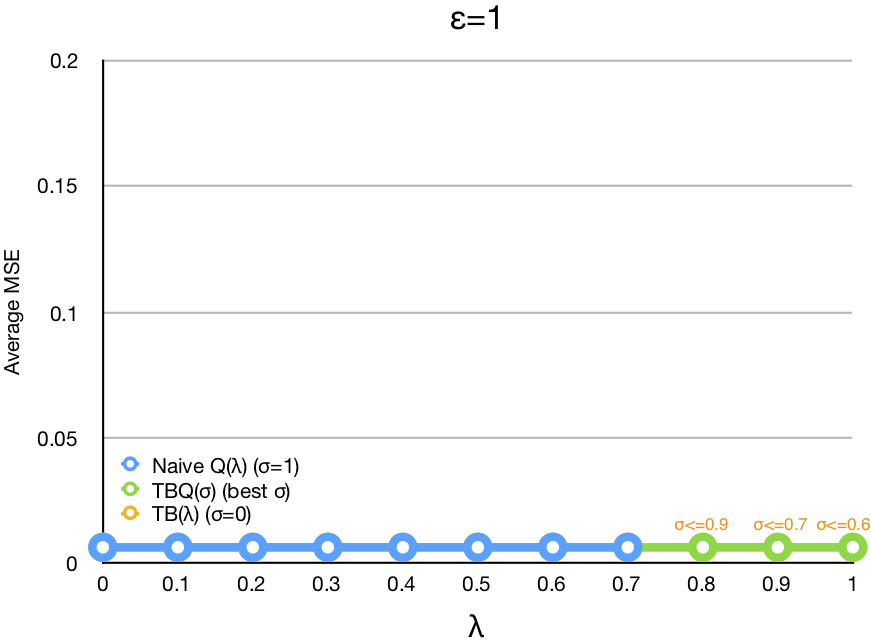}
	}
	\caption{$\lambda-\sigma$ relationship under different $\epsilon$ values} 
\end{figure}

Figure 1(a) shows that $\epsilon=0.1$ is too small for the agent to explore the whole environment. The agent can seldom reach the left terminal. In addition, since the exploratory action is also rarely taken, the MSEs of TB($\lambda$) between different $\lambda$ values vary a little. Naive Q($\lambda$) never cuts traces and enjoys the convergence when $\lambda\leq 0.6$. When $\lambda>0.6$, the Naive Q($\lambda$) diverges. The MSEs of TBQ($\sigma$) vary a little when Naive Q($\lambda$) converged. When $\lambda>0.6$, we can still tune $\sigma$ to reach a lower MSE. The best $\sigma$ of TBQ($\sigma$) decreases as the increase of $\lambda$. When $\epsilon=0.5$~(Figure 1(b)), we observe results similar to $\epsilon=0.1$ when $\lambda\leq 0.6$. When Naive Q($\lambda$) diverges, TBQ($\sigma$) can also benefit from learning from the full returns by adjusting a suitable $\sigma$. The MSE can also be reduced as well. When $\epsilon=1$, the behavior policy becomes completely random. The performance between TB($\lambda$) and Naive Q($\lambda$) is nearly the same when $\lambda\leq 0.7$. When $\lambda>0.7$ we can also adjust a suitable $\sigma$ to ensure the convergence of TBQ($\sigma$).

In this experiment, we observe that when $\epsilon\in[0,1]$, $\lambda\in[0,1]$, we can adjust a suitable $\sigma$ in order to learn from the full returns and avoid cutting traces too often as well. In practice, when $\lambda$ is close to 0, $\sigma$ can be set to 1 to make full use of the traces. When $\lambda$ is close to 1, $\sigma$ can be set to a small number near 0 to improve the efficiency of traces utilization.

\subsection{10$\times$10 Maze Environment}
The Maze environment is a 2-dimensional navigation task\footnote{We here use this version of gym-Maze environment:

https://github.com/MattChanTK/gym-maze.}. The agent's goal is to find the shortest path from start to the goal. For each state, the agent has 4 actions: go up, go down, turn left or turn right. If the path is blocked, the agent will stay at the current location. The reward is 1 when the agent reaches the goal, while at any intermediate state the agent gets reward -0.0001. Each episode is terminated if the agent reaches the goal, or the step count exceeds 2,000. To ensure adequate exploration and speed up the training process as well, we here adopt an $\epsilon-greedy$ policy as behavior policy and linearly decay the parameter $\epsilon$ from 1 to 0.1 by 0.02. In this experiment, we use the on-line backward version of TBQ($\sigma$). The learning rate $\alpha$ is tuned to 0.05. We here use 6 different $\sigma$ factors of TBQ($\sigma$): \{0, 0.2, 0.4, 0.6, 0.8, 1\}, and measure the average total steps of each episode. In addition, the results are averaged across 10 independent runs with fixed random seeds.

\begin{figure}
   \centering 
   \includegraphics[width = .5\textwidth]{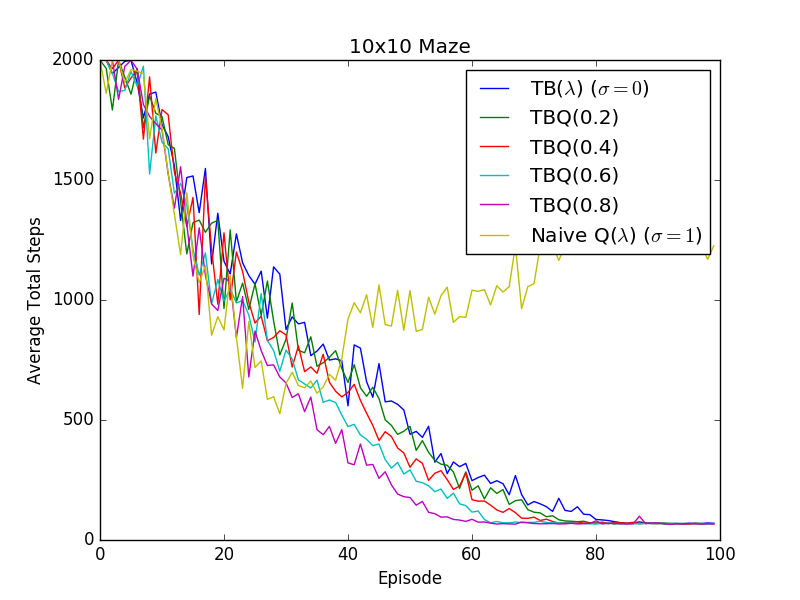}
   \caption{Averaged total steps of the Maze environment. TBQ($\sigma$) gradually accelerates the learning process when $\sigma$ varies from 0 to 0.8. However, Naive Q($\lambda$) diverges and cannot find the shortest path.}  
\end{figure} 

The result is illustrated in Figure 2. Since the shortest path of the maze is deterministic, TBQ($\sigma$) gradually accelerates the learning process when $\sigma$ varies from 0 to 0.8. However, Naive Q($\lambda$) diverges and cannot find the shortest path. The convergence speed of TBQ($\sigma$) reaches fastest at $\sigma=0.8$. The result shows that, in practice, we can accelerate the learning process by adjusting a suitable parameter $\sigma$ based on the TBQ($\sigma$) algorithm.

\subsection{TBQ($\sigma$) with Function Approximator}

We also evaluate TBQ($\sigma$) algorithm using neural networks as function approximator. With the help of deep Q-neworks~(DQN)~\cite{mnih-et-al:2015}, the offline version with a function approximator can be easily implemented. We here adopt online forward view for updating the parameters in the neural network. Unlike traditional DQN, we replay 4 consecutive sequences of samples with length of 8 for each update. We here evaluate TBQ($\sigma$) on CartPole problem~\cite{barto-et-al:1983}, and adopt the OpenAI Gym as the evaluation platform~\footnote{http: gym.openai.com}~\cite{brockman-et-al:2016}. In this setting, a pole is attached by an un-actuated joint to a cart, which can move along the track. The agent's goal is to prevent the pole from falling over with two actions controlling the cart: move left or right. Since the observation space is continuous, we adopt a two-layer neural network with 64 nodes in each layer to approximate the Q-value for the state action pairs. We use $\epsilon-greedy$ policy as behavior policy and exponentially decay the parameter $\epsilon$ from 1 to 0.1 by 0.995 to ensure adequate exploration. In addition, the target network parameters $\theta$ are updated using soft replacement~\cite{lillicrap-et-al:2015} according to the evaluation network parameter $\theta'$: $\theta\leftarrow\tau\theta+(1-\tau)\theta'$.

\begin{figure}
    \centering
    \includegraphics[width=.5\textwidth]{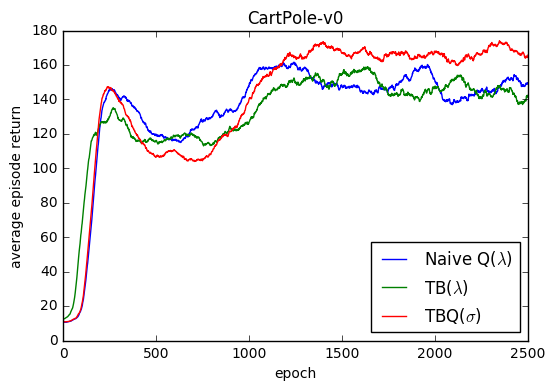}
    \caption{TBQ($\sigma$) with function approximator in CartPole environment. The exploring parameter $\epsilon$ of $\epsilon$- greedy policy decays from 1 to 0.1. To efficient utilize the traces under dynamic $\epsilon$, $\sigma$ is linearly decayed from 1 to 0.1 by step size 0.01. The result show that TBQ($\sigma$) outperforms both TB($\lambda$) and Naive Q($\lambda$).}
\end{figure}

\begin{table}[ht]
\centering
\begin{tabular}{l|r}
\hline
Parameter & Value \\
\hline
Discount factor & 0.99 \\
Initial exploration & 1 \\
Final exploration & 0.1 \\
Optimizer & Adam\cite{kingma-et-al:2014}\\
Initial learning rate & 0.001\\
Replay memory size & 20000\\
Replay start episode & 100\\
$\lambda$ & 1\\
$\tau$ & 0.001\\
\hline
\end{tabular}
\caption{Learning parameters for the neural network}
\end{table}

In this setting, in the beginning of the learning process the distance between target policy and behavior policy reach the maximum. When $\epsilon$ fades to 0.1, the two policy then become close. Therefore, to ensure convergence we here adopt a dynamic $\sigma$ linearly increase from 0.1 to 1 by stepsize 0.01. Other main learning parameters are listed in Table 1. the results are averaged across 5 independent runs with fixed random seeds. The result is showed in Figure 3. We also smooth the results with a right-centred moving average of 50 successive episodes. With a dynamic suitable $\sigma$, TBQ($\sigma$) outperforms TB($\lambda$) and Naive Q($\lambda$) in the CartPole problem. The result indicates that in practice, we can improve the learning by adjusting a suitable parameter $\sigma$ using TBQ($\sigma$) algorithm.

\section{Discussion and Conclusion}
In this paper, we propose a new off-policy learning method called TBQ($\sigma$) to define the degree of utilizing the off-policy traces. TBQ($\sigma$) unifies TB($\lambda$) and Naive Q($\lambda$). Theoretical analysis shows the contraction property of TBQ($\sigma$) in both policy evaluation and control. In addition, its convergence is proved for control setting. We also provide two versions of TBQ($\sigma$) control algorithm: online forward version TBQ$^F$($\sigma$) and online backward version TBQ$^B$($\sigma$). 

Comparing to TB($\lambda$), the proposed algorithm improves the efficiency of trace utilization when target policy is greedy. Comparing to Naive Q($\lambda$), our algorithm has relatively loose convergence requirement. Since the coefficient $c$ in our algorithm is less than 1, the variance of our algorithm is bounded~\cite{munos:2016}. Although we are not able to give further theoretical analysis between bootstrapping parameter $\lambda$ and degree of cutting traces $\sigma$ on convergence, we empirically show that the existing off-policy learning algorithms with eligibility traces can be improved and accelerated by adjusting a suitable trace-cutting degree parameter $\sigma$. The theoretical relationship between bootstrapping parameter $\lambda$ and $\sigma$ is remained for the future work.
%future works里面可加些内容，研究lambda-sigma-epsilon之间的关系。

\begin{acks}
The authors would like to thank the anonymous reviewers for their valuable comments and suggestions. This work is partly supported by National Key Research and Development Plan under Grant No. 2016YFB1001203, Zhejiang Provincial Natural Science Foundation of China (LR15F020001).

\end{acks}

%%%%%%%%%%%%%%%%%%%%%%%%%%%%%%%%%%%%%%%%%%%%%%%%%%%%%%%%%%%%%%%%%%%%%%%%%%%%%%%%%%%%%%%%%%%%%%%%%%%%%%%%%
%% bibliography: see CFP for number of permitted pages

\bibliographystyle{ACM-Reference-Format}  % do not change this line!
\balance  % do not change this line -- unless you manually balance your last page
\bibliography{reference}  % put name of your .bib file here

\end{document}